\documentclass[pdflatex,sn-mathphys-num]{sn-jnl}% Math and Physical Sciences Numbered Reference Style 
\usepackage{graphicx}%
\usepackage{multirow}%
\usepackage{amsmath,amssymb,amsfonts}%
\usepackage{amsthm}%
\usepackage{mathrsfs}%
\usepackage[title]{appendix}%
\usepackage{xcolor}%
\usepackage{textcomp}%
\usepackage{manyfoot}%
\usepackage{booktabs}%
\usepackage{algorithm}%
\usepackage{algorithmicx}%
\usepackage{algpseudocode}%
\usepackage{listings}%
\usepackage{listings}
\usepackage{color}
\usepackage{xcolor}
\usepackage{float}

\usepackage{amsthm}

\newtheorem{lemma}{Lemma}

\newtheorem{definition}{Definition}
\newtheorem{proposition}{Proposition}

\definecolor{dkblue}{rgb}{0,0.39,0}
\definecolor{gray}{rgb}{0.66,0.66,0.66}
\definecolor{mauve}{rgb}{0.91,0.33,0.50}
\definecolor{gold}{rgb}{1,0.84,0}

\newtheorem{remark}{Remark}%

\begin{document}

\title[Cybercrime Prediction via Geographically Weighted Learning]{Cybercrime Prediction via Geographically Weighted Learning}

\author*[1]{\fnm{Muhammad Al-Zafar} \sur{Khan}}\email{Muhammad.Al-ZafarKhan@zu.ac.ae}

\author[1,2]{\fnm{Jamal} \sur{Al-Karaki}}\email{Jamal.Al-Karaki@zu.ac.ae}

\author[1]{\fnm{Emad} \sur{Mahafzah}}\email{Emad.Mahafzah@zu.ac.ae}

\affil*[1]{\orgname{College of Interdisciplinary Studies (CIS), Zayed University}, \state{Abu Dhabi}, \country{UAE}}

\affil[2]{\orgdiv{College of Engineering}, \orgname{The Hashemite University}, \city{Zarqa}, \country{Jordan}}

\abstract{Inspired by the success of Geographically Weighted Regression and its accounting for spatial variations, we propose GeogGNN -- A graph neural network model that accounts for geographical latitude and longitudinal points. Using a synthetically generated dataset, we apply the algorithm for a 4-class classification problem in cybersecurity with seemingly realistic geographic coordinates centered in the Gulf Cooperation Council region. We demonstrate that it has higher accuracy than standard neural networks and convolutional neural networks that treat the coordinates as features. Encouraged by the speed-up in model accuracy by the GeogGNN model, we provide a general mathematical result that demonstrates that a geometrically weighted neural network will, in principle, always display higher accuracy in the classification of spatially dependent data by making use of spatial continuity and local averaging features.}

\keywords{Cybercrime Prediction, Geographically Weighted Learning, Geographically Weighted Graph Neural Networks (GeogGNNs)}

\footnotetext{Submitted to the International Jordanian Cybersecurity Conference 2024 (IJCC2024)}

\maketitle
\section{Introduction}

Geographically Weighted Regression (GWR) is an immensely powerful spatial analysis method in statistics for performing regression tasks in a non-Euclidean geometric setting that accounts for dimensional features of each datapoint where regression coefficients vary from one location to another. Introduced by Fotheringham et al. in \cite{brunsdon1996geographically}, with follow-up studies and monographs in \cite{brunsdon1998geographically,fotheringham2009geographically,comber2023route}, the method has been immensely successful in applications to Computational Geography and related areas. In addition, the method has seen widespread applications in Archaeology \cite{bevan2009modelling}, Cartography \cite{mennis2006mapping}, Transportation and Real Estate Studies \cite{cardozo2012application,dziauddin2015estimating,caset2020integrating}, Physiology and Health \cite{gilbert2011using,chakraborty2022exploring,ma2023influential,hassaan2021gis,mohammadi2023covid,corner2013modelling,melaku2022geographical}, Population Modeling and Migration Patterns \cite{lin2011using,wang2019examining}, Ecology \cite{windle2010exploring}, Urban and City Planning \cite{okwori2021spatial,taghipour2014application}, Economics and Labor Market Studies \cite{lewandowska2018geographically,craig2022social}, and many other application domains. 

The key subtly is that standard models do not account for the point-to-point spatial variation, namely given a set of features $\mathbf{X}=\left(x_{1},x_{2},\ldots x_{n}\right)$, the OLS model is
\begin{equation}
\hat{y}=\theta_{0}+\theta_{1}x_{1}+\theta_{2}x_{2}+\ldots+\theta_{n}x_{n}+\epsilon,
\end{equation}
with $\epsilon$ as the error term and where the optimal value of the coefficients that fit the model, $\boldsymbol{\theta}^{*}$, are given by
\begin{equation}
\boldsymbol{\theta}^{*}=\underset{\text{all}\;\boldsymbol{\theta}}{\arg\min}\;|\hat{y}-\theta_{0}-\boldsymbol{\theta}^{T}\mathbf{X}|^{2}.
\end{equation}
This standard paradigm fails to account for spatial variations at different points. Thus, the impetus for the development of GWR. 

Naturally, the next question in the progression of these models was to ask how these models could be extended for classification tasks results in the development of the Geographically Weighted Logistic Regression (GWLR) model for binary response variable problems \cite{nkeki2019geographically,zafri2023using}, Geographically Weighted Discriminant Analysis (GWDA) \cite{brunsdon2007geographically}, and now the current state-of-the-art fields of Spatial Machine Learning (SML) methods that encompass models such as Geographically Weighted Random Forest (GW-RF) \cite{luo2022socioeconomic,santos2019geographically,khan2022geographically,wu2024geographically,georganos2022forest}, Geographically Weighted SVMs (GW-SVMs) \cite{andris2013support,yang2023geographically}, Geographically Weighted Artificial Neural Networks (GW-ANNs) \cite{hagenauer2022geographically}, amongst others.  

As cybercrime becomes ever more sophisticated, the apprehension of cybercriminals becomes an ever-increasing enterprise among governments and law enforcement officials. Having the ability to associate a location to an IP address from where an attack took place is extremely powerful for the long-term goal of identifying hotspots, being able to categorize the nature of the cybercrime level, and identifying hotspots such as cybercrime offices and farms from where these attacks take place. As such, graph theory becomes an archetypical modeling approach because the points on a map can be represented as nodes, and weightings between their connections (edges) can be sought. However, standard graph ML  algorithms assume that the underlying fabric for which the modeling takes place is a flat space; thus, an extension to realistic coordinates, such as points of a map, is a natural next question. 

One such consideration is a re-definition of the adjacency matrix, borrowing from Gaussian and bisquare kernels of GWR, to give a more rich definition of the notion of distance between nodes. In this research, we consider a Gaussian-esque kernel where standard notion of the Euclidean norm still applies, however, it is attenuated by a scale parameter -- analogous to the bandwidth parameter in GWR -- which adjusts for localization.   

A natural question that arises is that based on the Universal Approximation Theorem, if NNs are a universal function approximation, why then does one not simply feed in the latitude and longitude coordinates as features and train the NN to learn based on this? As is evidenced in the experiments below, this decreases the accuracy of the model significantly, resulting in Proposition \ref{error reduction}. 

The paper is divided as follows:

In Sec. \ref{related work}, we discuss some important literature results on using GNNs to solve problems related to traditional and cybercrimes. 

In Sec. \ref{theory}, we provide the theoretical framework of the model we employed. 

In Sec. \ref{experiments}, we provide the results of the experiments we performed and benchmark our model against traditional NNs. 

In Sec. \ref{mathematical analysis}, we establish an important mathematical result. 

In Sec. \ref{conclusion}, we provide closing remarks on this research and ponder and reflect on the outcomes and implications. 

%====================================================================
\section{Related Work}\label{related work}

In \cite{nisha2023semantic}, the authors propose a hybrid semantic GNN (SGNN) that is combined with a Convolutional Neural Network (CNN) to address spam email detection via a binary classification scheme. It is shown that the leveraging of semantic information improves model accuracy. Similarly, in \cite{navarro2023deep}, the authors used Graph Convolutional Networks with directed edges (DDCN) to detect fraudulent transactions in finance. 

In \cite{guo2021deep}, the authors use GNNs that incorporate heterogeneous information in the form of stable and dynamic relationships between users and content to improve the accuracy of spam detection. The proposed model showed a small, but nonetheless significant, improvement in the detection accuracy. 

In \cite{kurshan2020financial}, the authors focus on the digital payments space and discuss the challenges and opportunities offered by GNNs. Similarly, in \cite{bilot2023graph} and in \cite{nicholls2021financial}, the authors provide a review of the state-of-the-art GNN-based methods. 

In addition to cybercrime, there has been a plethora of applications to urban crimes. These include the works in \cite{tekin2023crime,yang2023crimegnn,zhou2024hdm,hou2022integrated,roshankar2023spatio}.

Clearly, what is lacking in the literature is the accounting for spatial variation at graph nodes. In this research, we aim to fill that gap by demonstrating the use of latitude and longitudinal points and a re-definition of the connectivity matrix regarding Euclidean distance and accounting for the influence of localization.

%=====================================================================
\section{Theory: Geographically Weighted Graph Neural Networks (GeogGNNs)}\label{theory}

Standard Euclidean ML, in the context of classification problems, tries to fit a function $f:\mathbf{X}\to y$ from a feature space $\mathbf{X}=\left(x_{1},x_{2},\ldots,x_{n}\right)$ to a target variable $y$. However, the function $f$ does not account for spatial variations in the data. In the context of a geographical setting, these are the longitudes and latitudes of location points on a map. Thus, in order to successfully account for spatial variations, we need to account for the non-Euclidean geometrical effects at each point, namely, to find a function $g:\mathbf{X}(\mathcal{X},\mathcal{Y})\to y$ such that the feature space takes into consideration spatial variations at different coordinate points $\left(\mathcal{X},\mathcal{Y}\right)$. To this end, mathematical objects such as graphs $G=\left(V,E\right)$ composed of vertices $V\left\{v_{1},v_{2},\ldots v_{N}\right\}$ and edges $E=\left\{e_{i,j}:e_{i,j}=e_{j,i},1\leq i,j,\leq N,i\neq j\right\}$ forms a natural paradigm for modeling.

Given the graph $G=\left(V,E\right)$, with $V$ representing features in the context of ML. We define the adjacency matrix $\mathbf{A}$ as
\begin{equation}
A_{i,j}=
\begin{cases}
1,\quad\text{if there is an edge between nodes}\;i\;\text{and}\;j, \\
0,\quad\text{if there is no edge}.
\end{cases}
\end{equation}
$\mathbf{A}$ represents the strength of the relationship between the nodes. To each node, there is an associated feature vector
\begin{equation}
\mathbf{X}_{i}=
\begin{pmatrix}
x_{i,1} \\ x_{i,2} \\ \vdots \\ x_{i,n}
\end{pmatrix},
\end{equation}
where there are $n$ features and $1\leq i\leq |V|$, with $|V|$ being the cardinality of the vertex set.  

Thereafter, the node features are updated in the graph convolution layer by aggregating the features of neighboring nodes. Mathematically, this is accomplished as follows: Given layer $l$, the feature matrix update rule in layer $l+1$ is given by
\begin{equation}
\mathbf{H}_{\left(l+1\right)}=\sigma\left(\tilde{\mathbf{D}}^{-1/2}\tilde{\mathbf{A}}\tilde{\mathbf{D}}^{-1/2}\mathbf{H}_{\left(l\right)}\mathbf{W}_{\left(l\right)}\right),
\end{equation}
with $\mathbf{H}_{\left(0\right)}\leftarrow\mathbf{X}$, $\tilde{\mathbf{A}}=\mathbf{A}+\mathbf{I}$ is the adjacency matrix with self-loops that allows each node to consider its own features, $\tilde{\mathbf{D}}$ is the diagonal degree matrix of $\tilde{\mathbf{A}}$, $\mathbf{W}_{\left(l\right)}$ is the learnable weight matrix at layer $l$ and $\sigma$ is the nonlinear activation function, typically chosen to be ReLU. 

When working with geographical data, the standard definition of the adjacency matrix is not sufficient, and we define it in terms of the geographical kernel with a Gaussian profile
\begin{equation}
\begin{aligned}
A_{i,j}=&\;\exp\left(-\frac{d_{i,j}^{2}}{2\varphi^{2}}\right), \\
d_{i,j}=&\;||\mathbf{X}_{i}-\mathbf{X}_{j}||_{2},
\end{aligned}
\end{equation}
where $d_{i,j}$ is the Euclidean distance between nodes $i$ and $j$ and $\varphi$ is the scale of influence factor with the property that the smaller the value of $\varphi$, the more localized is the influence. 

We give a complete description of the training in Algorithm \ref{algo1}.

\begin{algorithm}[H]
\caption{GeogGNN}
\label{algo1}
\small{
\begin{algorithmic}[1]
\State \textbf{input:} feature matrix $\mathbf{X}\in\mathbb{R}^{N\times F}$ composed of $N$ nodes and $F$ features per node, label vector $\mathbf{y}\in\mathbb{R}^{N}$, latitude and longitude ordered pair $\left\{\left(\mathcal{X}_{i},\mathcal{Y}_{i}\right):1\leq i\leq N\right\}$, threshold distance $\lambda$
\State \textbf{initialize:} $\mathbf{H}_{0}\leftarrow\mathbf{X}$, normalize the features to the scale $\left[0,1\right]$ using standardization $x_{\text{norm}}=\left(x-\mu_{x}\right)/\sigma_{x}$ with $\mu_{x}$ and $\sigma_{x}$ being the mean and standard deviation of the respective feature
\Repeat 
\For {each pair of nodes $\left(v_{i},v_{j}\right)\in V, 1\leq i,j\leq N, i\neq j$}
\State calculate the geographical distance between nodes
\begin{equation*}
d_{i,j}=\sqrt{\left(\mathcal{X}_{i}-\mathcal{X}_{j}\right)^{2}+\left(\mathcal{Y}_{i}-\mathcal{Y}_{j}\right)^{2}}
\end{equation*}
\If {$d_{i,j}<\lambda$}
\State add an edge $e_{i,j}\in E$ between $v_{i}$ and $v_{j}$
\State $\mathbf{A}\ni A_{i,j}=1$
\EndIf
\State $\tilde{\mathbf{A}}\leftarrow\mathbf{A}+\mathbf{I}$ \Comment{allowing for self-loops in the adjacency matrix}
\State calculate the diagonal degree matrix of $\tilde{\mathbf{A}}=[\tilde{A}]_{i,j}$
\begin{equation*}
\begin{aligned}
\tilde{\mathbf{D}}=&\;\text{diag}\left(D_{1},D_{2},\ldots,D_{N}\right), \\
D_{i}=&\;\sum_{j=1}^{N}\tilde{A}_{i,j}
\end{aligned}
\end{equation*}
\For {each node $v\in V$} 
\For {each layer $l=1:l_{N}$}
\State aggregate feature information from neighbors and apply 
\begin{equation*}
\mathbf{H}_{\left(l+1\right)}=\sigma\left(\tilde{\mathbf{D}}^{-1/2}\tilde{\mathbf{A}}\tilde{\mathbf{D}}^{-1/2}\mathbf{H}_{\left(l\right)}\mathbf{W}_{\left(l\right)}\right)
\end{equation*}
\State calculate probabilities of each class
\begin{equation*}
P=\text{softmax}\left(\tilde{\mathbf{D}}^{-1/2}\tilde{\mathbf{A}}\tilde{\mathbf{D}}^{-1/2}\mathbf{H}_{\left(l_{N}\right)}\mathbf{W}_{\left(l_{N}\right)}\right)
\end{equation*}
\State calculate cross-entropy loss
\begin{equation*}
J=-\frac{1}{N}\sum_{v\in V}\sum_{c\in C}y_{v,c}\log P_{v,c}
\end{equation*}
\EndFor
\EndFor
\EndFor
\Until {$J\to 0$}
\end{algorithmic}
}
\end{algorithm}

The following metrics were used to assess the model's performance: F1 score, precision, recall, log-loss, AUC-ROC, and AUC-PR. As these are standard metrics and can be found in any textbook on ML, a restatement is avoided here. However, in interpreting the results, these are alluded to and explained. 

%=========================================================================

\section{Experiments}\label{experiments}

Following Algorithm \ref{algo1}, the model was built, and the results are encapsulated in Figs. \ref{fig1} and \ref{fig2}. The model was benchmarked against a standard NN and a standard CNN model; the results of the former's training are shown in Figs. \ref{fig3}--\ref{fig5}, and the results of the latter's training are shown in Figs. \ref{fig8}--\ref{fig6}.  

\begin{figure}[H] % GeogGNN
    \centering
    \includegraphics[width=1.2\linewidth]{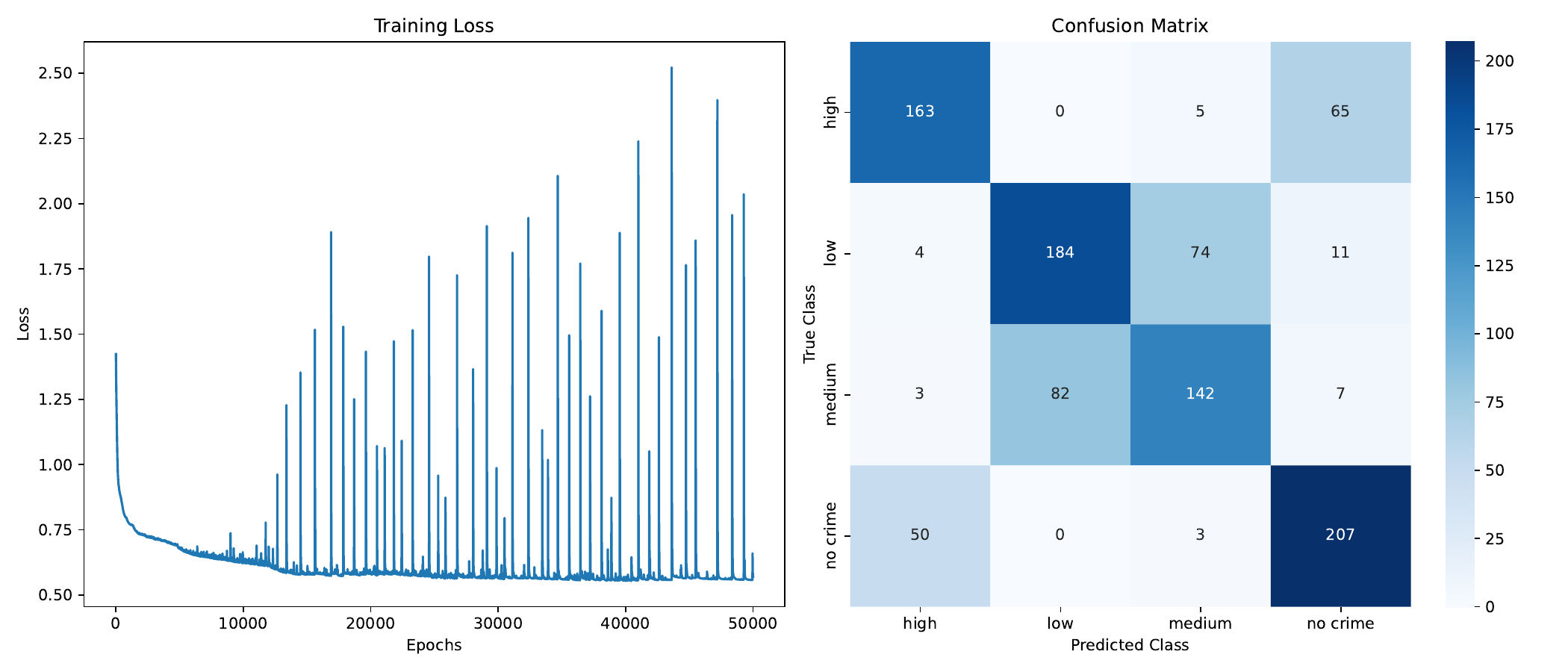}
    \caption{Loss function and confusion matrix for the model for the GeogGNN model.}
    \label{fig1}
\end{figure}

\begin{figure}[H] % GeogGNN
    \centering
    \includegraphics[width=1.2\linewidth]{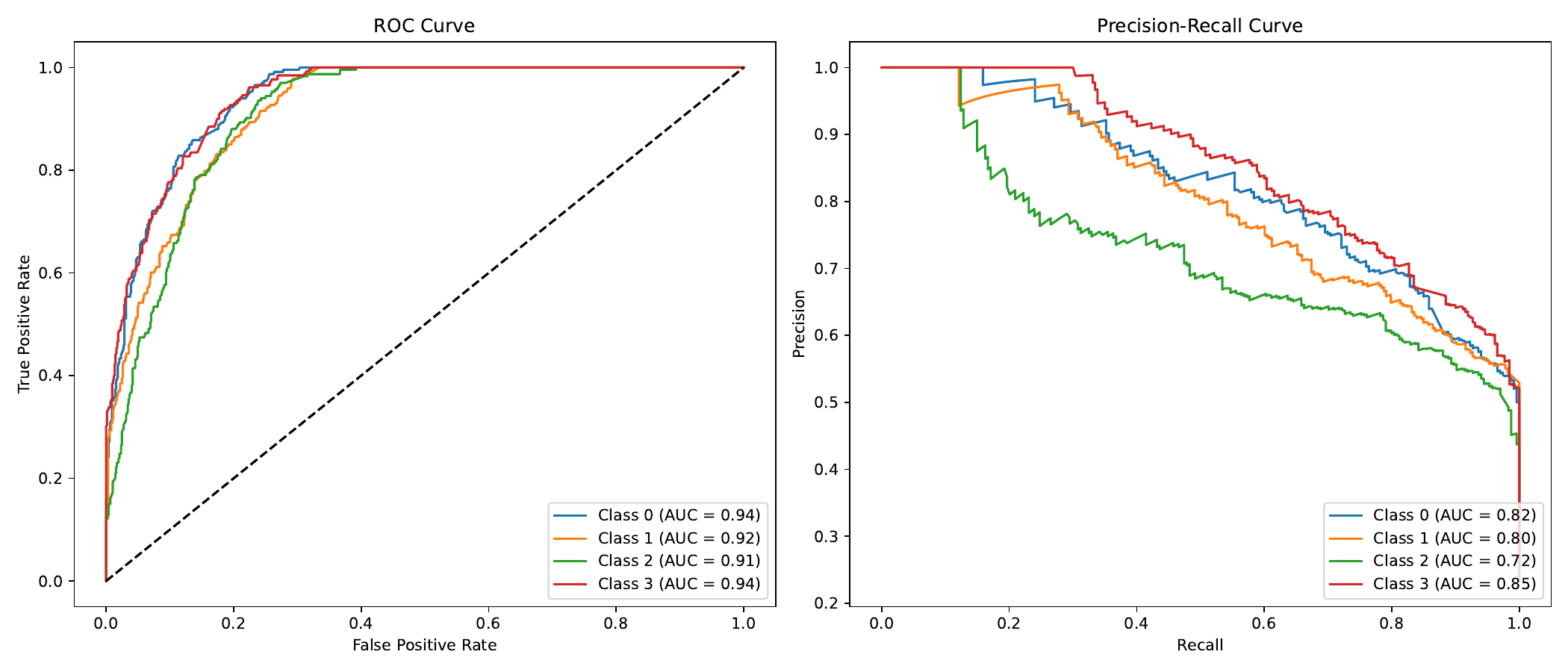}
    \caption{ROC and PR curves for each class in the model GeogGNN model.}
    \label{fig2}
\end{figure}

\begin{figure}[H] % standard NN
    \centering
    \includegraphics[width=1.0\linewidth]{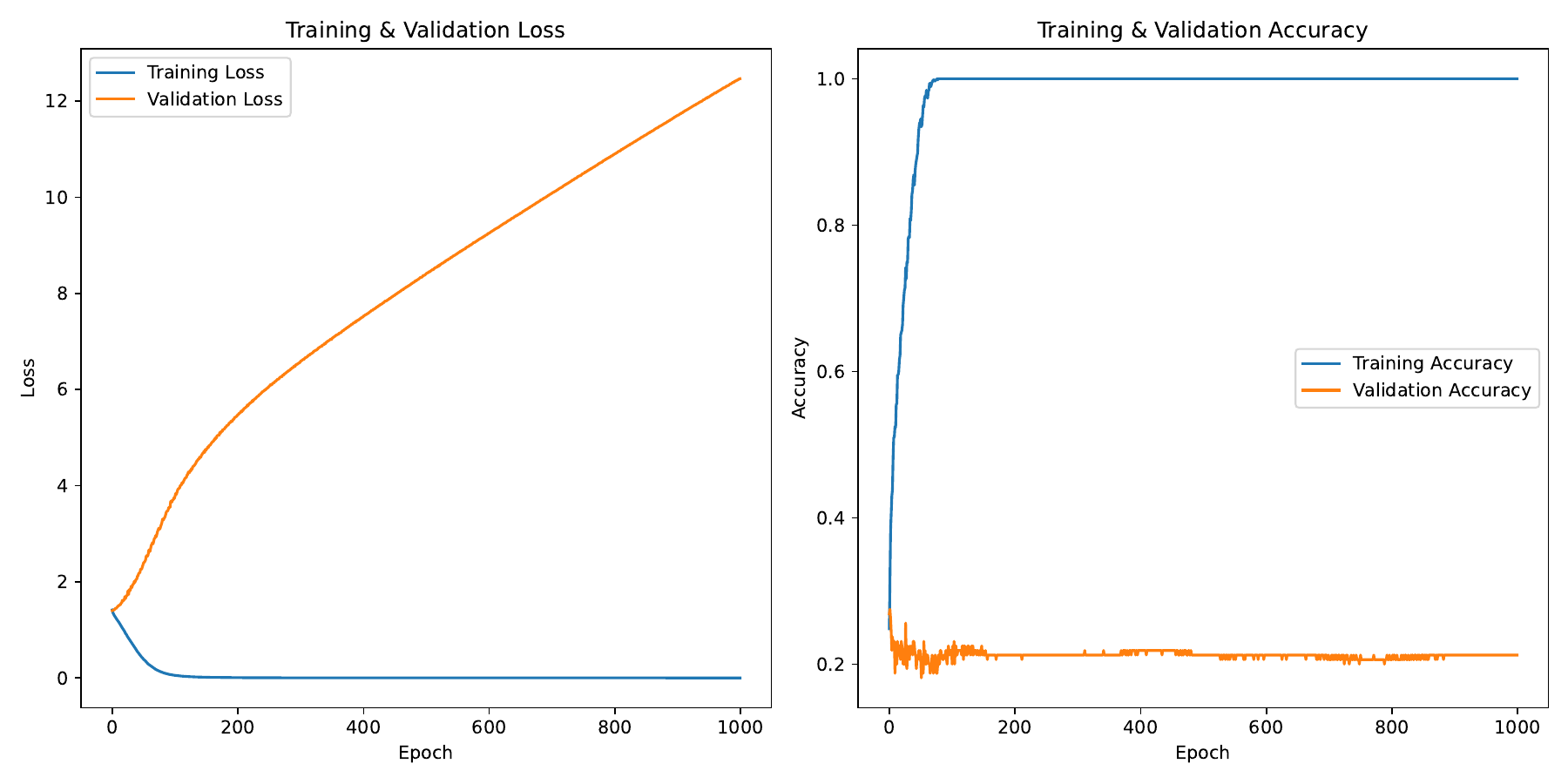}
    \caption{Loss function and training and validation accuracy for the standard neural network model.}
    \label{fig3}
\end{figure}

\begin{figure}[H] % standard NN
    \centering
    \includegraphics[width=0.8\linewidth]{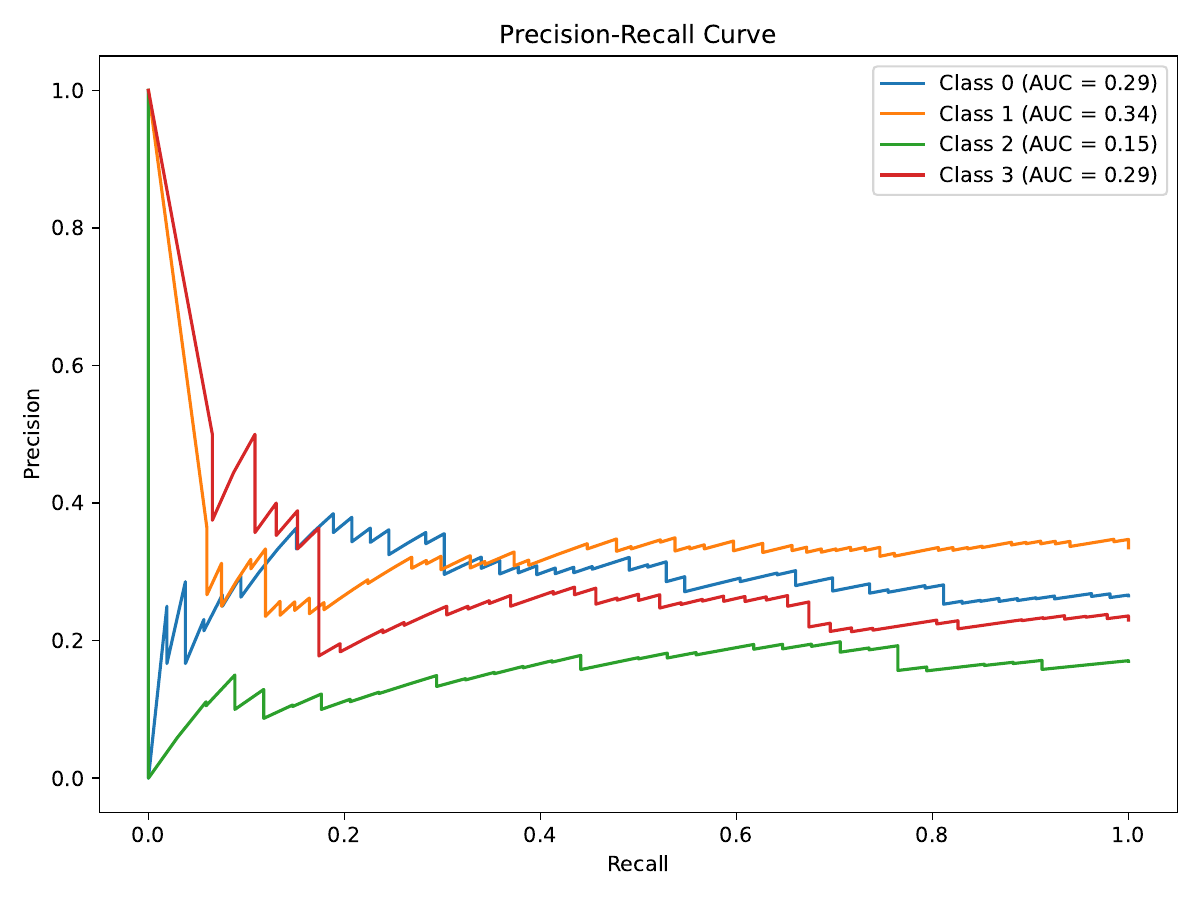}
    \caption{Precision-recall curves for the standard neural network model.}
    \label{fig4}
\end{figure}

\begin{figure}[H] % standard NN
    \centering
    \includegraphics[width=1.0\linewidth]{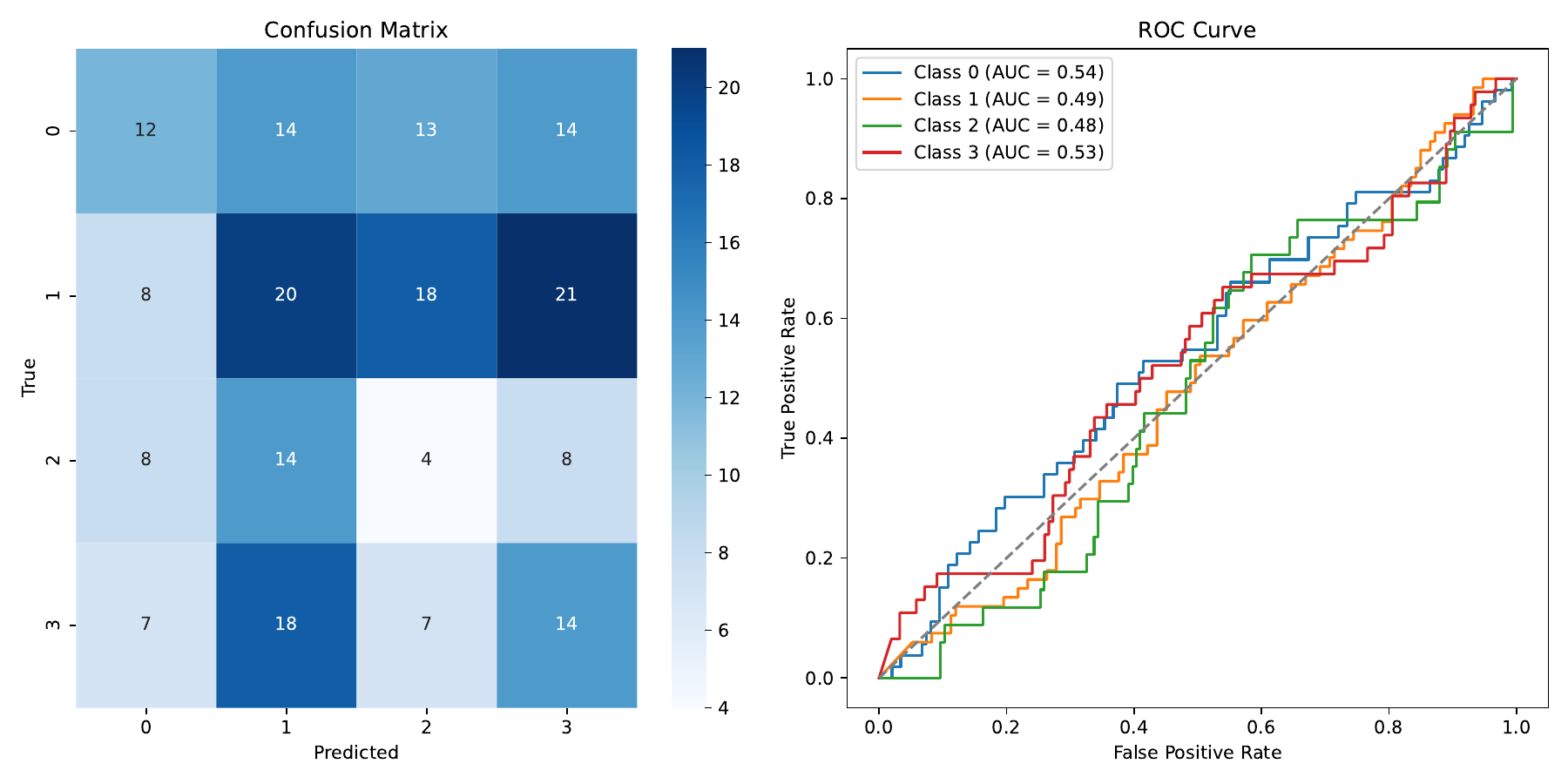}
    \caption{Confusion matrix and ROC curves for the standard neural network model.}
    \label{fig5}
\end{figure}

\begin{figure}[H]
    \centering
    \includegraphics[width=1.0\linewidth]{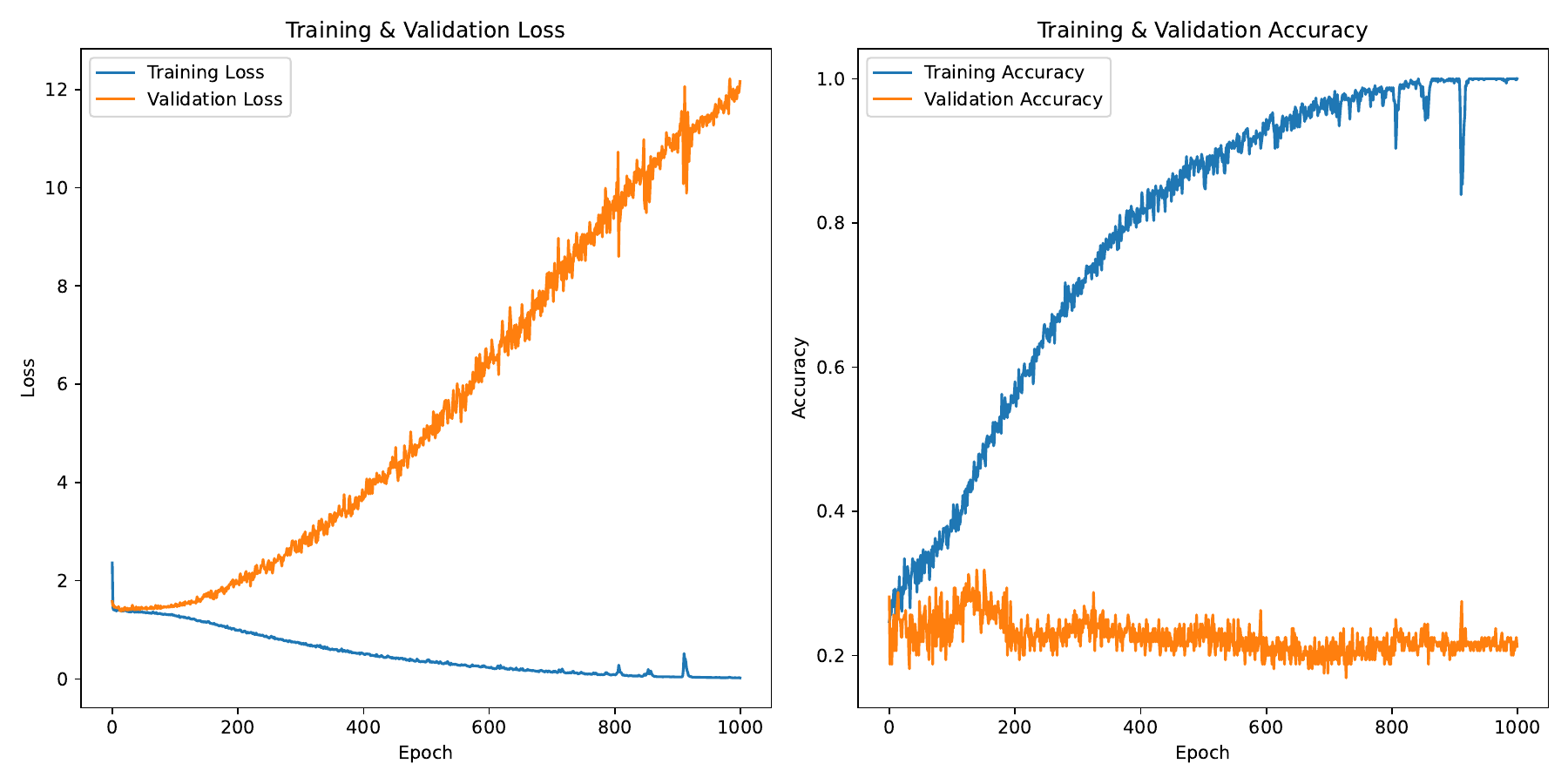}
    \caption{Loss function and training and validation accuracy for the convolutional neural network model.}
    \label{fig8}
\end{figure}

\begin{figure}[H]
    \centering
    \includegraphics[width=0.9\linewidth]{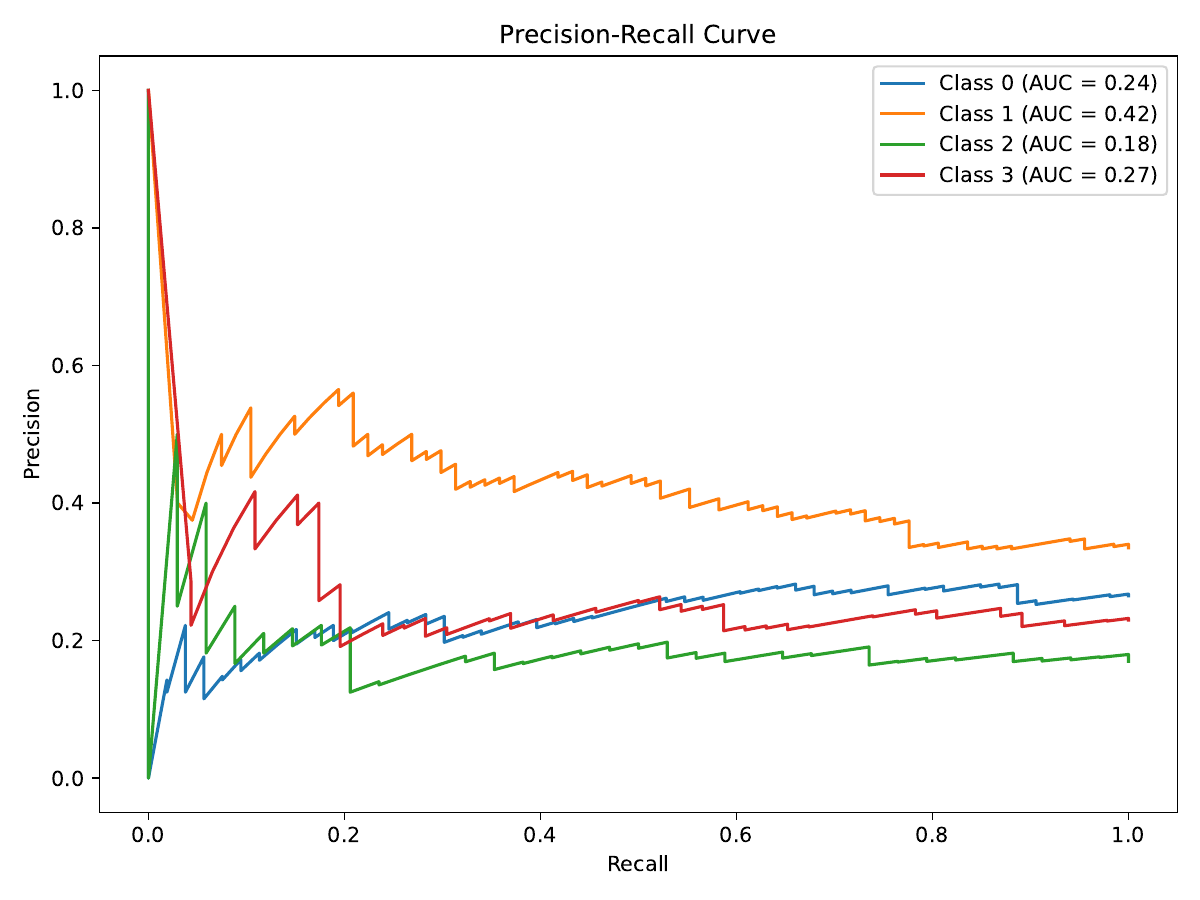}
    \caption{Precision-recall curves for the convolutional neural network model.}
    \label{fig7}
\end{figure}

\begin{figure}[H] % cnn
    \centering
    \includegraphics[width=1.1\linewidth]{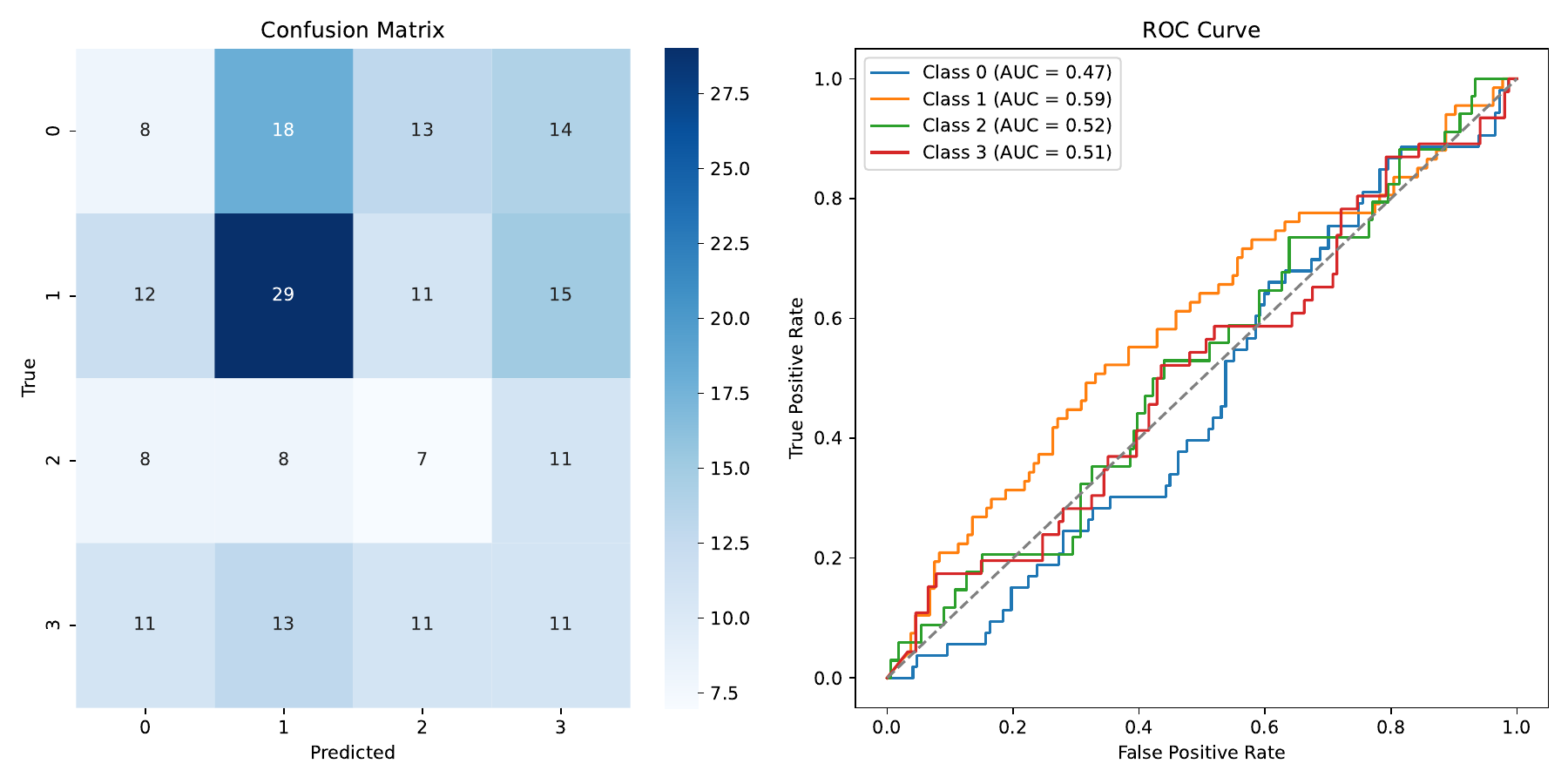}
    \caption{Confusion matrix and ROC curves for the convolutional neural network model.}
    \label{fig6}
\end{figure}

The violently spiking loss function in Fig. \ref{fig1} with increasing loss values from around $12\;000$ epochs is an artifact of training the model for so many epochs. A workaround ``cheat'' is to only train the model for around $10\;000$ epochs and report a nice smooth monotonically decreasing profile, however, by training the model for more epochs we were able to reduce the loss further, although we acknowledge that it is marginal.

We benchmark the GeogGNN model against a standard NN and  CNN in Tab. \ref{tab3}. 

\begin{table}[htpb]
\centering
\begin{tabular}{llllllll}
\hline 
\textbf{Model} &\textbf{Accuracy} &\textbf{F1} &\textbf{Precision} &\textbf{Recall} &\textbf{Log-Loss} &\textbf{AUC-ROC} &\textbf{AUC-PR} \\
\hline 
GeogGNN &0.696 &0.6959 &0.6954 &0.6960 &0.5723 &Class 0: 0.9411 &Class 0: 0.8239 \\
 & & & & & &Class 1: 0.9183 &Class 1: 0.8021 \\
 & & & & & &Class 2: 0.9102  &Class 2: 0.7219 \\
 & & & & & &Class 3: 0.9421 &Class 3: 0.8532 \\
\hline 
Standard NN &0.2500 &0.2534 &0.2651 &0.2500 &N/A &Class 0: 0.54 &Class 0: 0.29 \\
 & & & & & &Class 1: 0.49 &Class 1: 0.34 \\
 & & & & & &Class 2: 0.48 &Class 2: 0.15 \\
 & & & & & &Class 3: 0.53 &Class 3: 0.29 \\
\hline
CNN &0.2750 &0.2735 &0.2752 &0.2750 &N/A &Class 0: 0.47  & Class 0: 0.24 \\
 & & & & & &Class 1: 0.59 &Class 1: 0.42 \\
 & & & & & &Class 2: 0.52  &Class 2: 0.18 \\
 & & & & & &Class 3: 0.51 &Class 3: 0.27 \\ 
\hline  
\end{tabular}
\caption{Comparison of the metrics of the three models.}
\label{tab3}
\end{table}

From Tab. \ref{tab3}, we observe that the GeogGNN significantly outperforms the standard NN and CNN models across all metrics. Further, GeogGNN exhibits strong performance across all classes, with high accuracy, precision, recall, and AUC-ROC/AUC-PR scores. This indicates that the model is highly effective in distinguishing between different classes. This provides the impetus for Proposition \ref{error reduction} whereby we try to speak about the performance of NN architectures that take into account geographical weightings. 

The Standard NN performs poorly across all metrics, suggesting that it is not suitable for this classification task. We suspect that the performance is poor because the latitude and longitude coordinates are treated as features. The low accuracy and F1 score indicate that the model struggles to make correct predictions. The AUC-ROC and AUC-PR scores for different classes show significant disparities, indicating that the model may be biased toward certain classes. This cannot be attributed to a class imbalance because, as shown in Fig. \ref{fig10}, the classes are approximately balanced. 

The CNN's performance is better than the standard NN but still significantly worse than the GeogGNN. Further, the CNN model struggles to distinguish between different classes, as evidenced by the low AUC-ROC and AUC-PR scores.

The above analysis shows that the GeogGNN model is a powerful tool for handling complex, geographically distributed data. The model can effectively capture the underlying structure and dependencies between datapoints, as demonstrated with this synthetically-generated dataset. Traditional ML models like standard NN and CNN that treat the geographic coordinates as features may not be suitable for tasks involving complex relationships.

%=======================================================================

\section{Mathematical Analysis}\label{mathematical analysis}
In this section, we prove a fundamental result: GWNNs have a lower error than standard neural networks. By first establishing some definitions, we then provide some lemmas and thereafter prove our main result. Lastly, we comment on why this holds true. 

\begin{definition}\label{latitude and longitude}
Let $\mathcal{L}\subset\mathbb{R}^{2}$ be a geographic coordinate space with each $\mathcal{L}\ni\ell=\left(\text{lat},\text{long}\right)$ corresponding to latitude and longitude respectively.      
\end{definition}

\begin{definition}
Let $\mathcal{C}=\left[C\right]=\left\{1,2,\ldots,C\right\}$ be the set of class labels.     
\end{definition}

\begin{definition}
Let $h^{*}:\mathcal{L}\to\mathbb{R}^{C}$ be a true classification function with $\ell\mapsto c$ for all $\ell\in \mathcal{L}$ and $c\in\mathcal{C}$. 
\end{definition}

\begin{definition}
Let $h:\mathcal{L}\to\mathbb{R}^{C}$ that approximates $h^{*}$ by generating an output vector of class probabilities for each point $\ell\in\mathcal{L}$ noting that $h$ does not explicitly incorporate any spatial dependence among points in $\ell\in\mathcal{L}$. 
\end{definition}

\begin{definition}
A Geographically Weighted Neural Network (GWNN) modifies $h$ such that it includes a spatial weighting function $\mathcal{W}:\mathcal{L}\times\mathcal{L}\to\left[0,1\right]$ that adjusts the influence of each datapoint based on its proximity to other points in $\mathcal{L}$. This results in a weighted network function $h_{\mathcal{W}}$ defined as
\begin{equation*}
h_{\mathcal{W}}\overset{.}{=}\int_{\mathcal{L}}K\left(||\ell-\ell'||\right)f(\ell')\;\text{d}\ell,
\end{equation*}
where $K\left(||\ell-\ell'||\right)=\mathcal{W}(\ell,\ell')$ is the kernel centered at $\ell$. 
\end{definition}

\begin{definition}\label{error kernel}
Let $\mathcal{E}(\ell)$ be the error functional that measures the discrepancy between the true classification function $h^{*}$ and the approximated function $h$ given by
\begin{equation*}
\mathcal{E}(h)\overset{.}{=}\int_{\mathcal{L}}||h(\ell)-h^{*}(\ell)||^{2}\;\text{d}\ell. 
\end{equation*}
Similarly, we define the error of the GWNN as
\begin{equation*}
\mathcal{E}(h_{\mathcal{W}})\overset{.}{=}\int_{\mathcal{L}}||h_{\mathcal{W}}(\ell)-h_{\mathcal{W}}^{*}(\ell)||^{2}\;\text{d}\ell. 
\end{equation*}
\end{definition}

We want to show that $\mathcal{E}(h_{\mathcal{W}})<\mathcal{E}(h)$. This shows that the GWNN provides a lower classification error than the standard NN on spatial data. Before doing so, we establish two important results that will feed into out proof. 

\begin{lemma}\label{spatial continuity}
(Spatial continuity of $h^{*}$). Let $\left(L,d\right)$ be a metric space endowed with the Euclidean notion of distance, i.e., $d=||\ell-\ell'||$ for $\ell,\ell'\in L$. Then, points that are geographically close (often) belong to the same class, and $h^{*}$ varies smoothly over $L$.    
\end{lemma}

\begin{proof}
Let $\varepsilon>0$ and $h^{*},h$ be the classification and approximate $h^{*}$ functions respectively. Then, there exists $\delta>0$ such that $d(l,l')<\delta$. This implies that
\begin{equation*}
|h^{*}(\ell)-h(\ell)|<\varepsilon.
\end{equation*}
This shows that points that are close together are WLOG in the same class, whereas those that are further apart are not in the same class. In addition, $h^{*}$ is continuous, and thus, it is inferred as smooth.
\end{proof}

\begin{lemma}\label{local averaging}
(Local averaging of $h^{*}$ around $\ell$). For the geographically weighted neural network $h_{\mathcal{W}}$ we define the exponentially decaying kernel with Gaussian profile
\begin{equation*}
\mathcal{W}(\ell,\ell')=K(||\ell-\ell'||)-e^{-\xi d^{2}}, \forall\;\ell,\ell'\in L,    
\end{equation*}
where $\xi$ is the decay parameter constant. Then
\begin{enumerate}
\item The spatial weighting in $h_{\mathcal{W}}$ produces a local averaging effect on $h^{*}$ around point $\ell$. 
\item This averaging effect reduces the variance of $h_{\mathcal{W}}$ compared to $h$ in regions where $h^{*}$ varies smoothly due to spatial continuity. 
\end{enumerate}
\end{lemma}

\begin{proof}
The mean-squared error between $h_{\mathcal{W}}$ and $h^{*}$ is given by
\begin{equation*}
\mathbb{E}\left\{\left[h_{\mathcal{W}}(\ell)-h^{*}(\ell)\right]^{2}\right\}\overset{.}{=}\mathbb{E}\left\{\int_{L}\left[K(||\ell-\ell'||)\left(h(\ell')-h^{*}(\ell)\right)\;\text{d}\ell\right]^{2}\right\}.
\end{equation*}
By Lemma \ref{spatial continuity} above, $h^{*}$ is continuous. This implies that the kernel $K(||\ell-\ell'||)$ concentrates the integral on points $\ell'$ near $\ell$. Thus, this reduces the variance of $h_{\mathcal{W}}$ around $h^{*}$. 
\end{proof}

\begin{proposition}\label{error reduction}
For a true classification function $h^{*}$ that is spatially continuous, a GWNN $h_{\mathcal{W}}$ yields a lower error than a standard neural network $h$.    
\end{proposition}

\begin{proof}
We want to show that $\mathcal{E}(h_{\mathcal{W}})<\mathcal{E}(h)$. By definition,
\begin{equation*}
\mathcal{E}(h)\overset{.}{=}\int_{L}||h(\ell)-h^{*}(\ell)||^{2}\;\text{d}\ell.
\end{equation*}
With spatial weighting, $h$ lacks the ability to exploit the spatial continuity of $h^{*}$. This leads to abrupt classification perturbations even for geographical points that are close. Now, the error for $h_{\mathcal{W}}$ is given by
\begin{equation*}
\mathcal{E}(h_{\mathcal{W}})=\int_{L}\left|\left|\int_{L}K\left(\ell-\ell'\right)\left[h(\ell')-h^{*}(\ell)\right]\;\text{d}\ell'\right|\right|^{2}\;\text{d}\ell.     
\end{equation*}
By Lemma \ref{local averaging}, spatial weighting provides local averaging around $\ell$. This respects the spatial continuity of $h^{*}$ and therefore reduces errors in regions where $h^{*}$ has a smooth variation. Thus, $\mathcal{E}(h_{\mathcal{W}})<\mathcal{E}(h)$ due to the localized smoothing effect.  
\end{proof}

\begin{remark}
The implications of Proposition \ref{error reduction} is that a GWNN is theoretically more accurate for classifying spatial data where classes exhibit spatial continuity. Further, Proposition \ref{error reduction} also demonstrates that GWNNs achieve better accuracy on spatial classification tasks because they better exploit the intrinsic spatial correlation present in the data.
\end{remark}

While our result may present a major stepping stone to demonstrating why spatial coordinates treated as features produce less accuracy than those that are geographically weighted, it may be possible that our result is falsified through experimentation. Thus, as a ``play it safe'' idea, we state our result as a proposition as opposed to a theorem. 

%=====================================================================

\section{Conclusion}\label{conclusion}
We have successfully shown that our proposed GeogGNN model outperforms standard NNs and CNNs in classification tasks by using a synthetically generated dataset. The next step would be to apply the model to real-world data and assess its performance. We have discovered that the way not to proceed is to not treat the geographical coordinates as features but rather to assign a geographical weighting scheme to them to make the model richer. 

We have also shown theoretically that NN models with geographic weighting have better performance than those that do not account for geographic variations among datapoints. This important result will be extended to other ML model architectures in future work. 

Quantum GNNs (QGNNs) have shown immense promise and the ability to outperform classical GNNs as illustrated in studies like \cite{innan2024financial}. It would be interesting to ascertain how GeogGNN fairs against a QGNN and to extend GeogGNNs to quantum GeogGNNs, QGeogGNNs, which will be considered as an offshoot project of this work, together with deriving a quantum version of Proposition \ref{error reduction}.

%====================================================================

\section*{Declarations}

\begin{itemize}
\item \textbf{Funding:} J.A.K. and M.A.Z. acknowledge that this research is supported by grant number 23070, provided by Zayed University and the government of the UAE.
\item \textbf{Conflict of interest/Competing interests:} The authors declare that there are no conflicts of interest. 
\item \textbf{Ethics approval and consent to participate:} None required.
\item \textbf{Consent for publication:} The authors grant full consent to the journal to publish this article.
\item \textbf{Data availability:} N/A
\item \textbf{Materials availability:} N/A
\item \textbf{Code availability:} N/A  
\item \textbf{Author contribution:} All authors have contributed equally to this research.
\end{itemize}

\end{document}